%% file: main.tex
\pdfoutput=1

\documentclass[11pt]{article}
\usepackage[utf8]{inputenc} 
\usepackage[final]{acl}

\usepackage{times}
\usepackage{latexsym}
\usepackage{amsmath}
\usepackage{amsfonts}
\usepackage{amssymb}
\usepackage{amsthm}
\usepackage{bm}

\usepackage[algo2e]{algorithm2e} 
\usepackage{algorithm}
\newtheorem{theorem}{Theorem}

\usepackage[T1]{fontenc}

\usepackage[utf8]{inputenc}

\usepackage{microtype}

\usepackage{inconsolata}

\usepackage{graphicx}

\usepackage{subcaption}
\usepackage{multirow}
\usepackage{paralist}
\usepackage{booktabs}
\usepackage{tabularx}
\usepackage{float}
\newcommand{\name}{\textsc{Ketchup}}

\title{\name: K-Step Return Estimation\\ for Sequential Knowledge Distillation}

\author{%
  Jiabin Fan, Guoqing Luo, Michael Bowling$^\ddagger$, Lili Mou$^\ddagger$\\
  Dept. Computing Science, Alberta Machine Intelligence Institute (Amii) \\
  University of Alberta, Canada\\
  $^\ddagger$Canada CIFAR AI Chair, Amii\\
  \texttt{\{jiabin,gluo,mbowling\}@ualberta.ca, doublepower.mou@gmail.com}
}

\begin{document}
\maketitle
\begin{abstract}
Reinforcement Learning (RL)-based knowledge distillation (KD) is increasingly used to train language models for text generation. However, existing methods suffer from high variance caused by long action chains during sampling. To address this, we propose a novel block-wise return induction approach (called \name) that mitigates the high variance issue and stabilizes the training process. 
Our idea is to apply the Bellman Optimality Equation inversely to each $K$-step block segmented student's explored trajectories, and thus induce a total reward for all blocks from the teacher model, serving as the policy-gradient training signal.
Theoretical analysis shows that our \name{} reduces the variance of the gradient estimates, thus leading to improved RL optimization. Empirical evaluation on three text generation tasks demonstrates that our approach yields superior performance in both standard task metrics and large language model (LLM)-based evaluation, which suggests that our \name{} offers a promising direction for enhancing RL-based on-policy Knowledge Distillation in LLM research. \footnote{Our code is released at \url{https://github.com/MANGA-UOFA/KETCHUP}}
\end{abstract}

\section{Introduction}

Knowledge distillation \cite[KD;][]{hinton2015distillingknowledgeneuralnetwork} refers to training a (typically) small student model from a teacher's output. KD has been increasingly important in the LLM era, as larger models achieve higher performance \cite{kaplan2020scalinglawsneurallanguage} but are more difficult to deploy in low-resource scenarios.

KD approaches can be generally categorized into two types: intermediate-layer matching and prediction matching. Intermediate-layer matching aims to match the student's and teacher's hidden states, encouraging the student to mimic the teacher's behavior layer by layer~\cite{sun-etal-2019-patient,jiao-etal-2020-tinybert,wang-etal-2021-minilmv2}. Prediction matching informs the student of the task to solve, typically by minimizing the divergence of output distributions \cite{kim-rush-2016-sequence,wen-etal-2023-f}.

Classic KD for text generation suffers from the exposure bias problem~\cite{10.5555/2969239.2969370}, as the student learns word by word following the teacher's or ground truth's prefix, without accounting for its own previous predictions. RL alleviates this issue by enabling the student to learn through exploration. \citet{hao2022teacher} induce a step-wise reward function from a language model trained in a supervised way. Building on this, \citet{li-etal-2024-llmr} apply RL to text generation KD, where a student model is trained by the REINFORCE algorithm \cite{williams1992simple} maximizing the cumulative reward suggested by the teacher.

However, REINFORCE is known to suffer from high variance because it estimates the gradient by sampled trajectories (i.e., sequences), which can vary significantly  \cite{sutton2018reinforcement}. This issue is further exacerbated in text generation scenarios due to the large action space (i.e., vocabulary size), resulting in unstable learning.

In this paper, we propose \name,  a novel $\boldsymbol {K}$\!-step return \textbf{E}stimation \textbf{T}e\textbf{CH}nique to \textbf{U}pdate \textbf{P}olicy for RL-based knowledge distillation. Our work is inspired by \citet{li-etal-2024-llmr}, who derive 
a Q-value function from the teacher's policy (next-token probabilities) and induce a 
reward function based on the Bellman Optimality Equation \citep{bellman1952theory}. 
In our approach, we break the long sampled trajectory of the student model into blocks of $K$ consecutive actions. For each block, we repeatedly apply the inverse of the Bellman Optimality Equation and induce a total reward for the block from the teacher model. Then, we sum the total rewards for all blocks as an approximate return (i.e., the total reward for a whole trajectory). We use such a block-wise approximated return as the RL training signal and update the student model with standard policy gradient~\citep{sutton1999policy}.
Theoretical analysis shows that our \name\ reduces the variance of the total reward, thus effectively mitigating the high variance issue of RL-based text generation KD.

In essence, our approach is a REINFORCE-with-baseline~\citep{williams1992simple,sutton2018reinforcement} variant that introduces an auxiliary term (called a \textit{baseline}) to the return, which oftentimes stabilizes RL training~\citep{sutton2000policy,greensmith2004variance,thomas2017action}. Traditional approaches have a baseline that solely depends on the sampled data (e.g., the mean return of all samples), but this introduces additional noise to the training process if the sampled set is small or not representative. Later, researchers extended REINFORCE-with-baseline by developing Actor--Critic (AC) algorithms that learn a critic (i.e., estimated state-value function) to stabilize training.  Our method also extends traditional REINFORCE-with-baseline but differs from AC: we derive a baseline term based on block-wise Bellman Optimality from the teacher model; thus, we do not need to train a cumbersome state-value function as AC algorithms do.

We evaluate our approach on three text generation datasets categorized into different domains: XSum \cite{narayan-etal-2018-xsum} for summarization,  the Europarl corpora \cite{koehn2005europarl} for machine translation, and GSM8K \cite{cobbe2021training} for arithmetic reasoning. Experiments show that our proposed \name{} consistently achieves an add-on performance improvement for RL-based KD \cite{li-etal-2024-llmr}. We also conduct an in-depth analysis showing that our \name{} demonstrates lower variance and converges better than~\citet{li-etal-2024-llmr}, i.e., achieving a higher return and being more stable.

\section{Methodology}
\subsection{RL Formulation of Text Generation}
Text generation can be formulated as an undiscounted Markov Decision Process (MDP) 
with tuple $(\mathcal{S}, \mathcal{A}, T, r)$. The \textit{state} space $\mathcal{S}$ 
includes all possible (sub)sequences and each of them is represented by $\mathbf y_{<t}$ for some time step $t$; notice that text generation may also depend on an input sequence, which is omitted here. The \textit{action} $a_t\in\mathcal{A}$ at step $t$ corresponds to the next 
token $\mathrm y_t$ from the vocabulary $\mathcal{V}$. The \textit{state transition} $T$ is a deterministic process in text generation, as $s_{t+1}$ is essentially the concatenation of $s_t$ and the newly generated word $a_t$. The \textit{reward} function 
$r: \mathcal{S} \times \mathcal{A} \to \mathbb{R}$ provides feedback based 
on $(s_t,a_t)$. The goal of RL is to find a \textit{policy} (distribution over actions) to maximize the expected \textit{return} (cumulative rewards). 

A key challenge in applying RL to text generation is the lack of well-defined step-wise
reward functions. To address this, \citet{hao2022teacher} and \citet{li-etal-2024-llmr} assume that a language model generates the next word from a Boltzmann distribution based on the \textit{Q-value function},\footnote{The 
Q-value function estimates the expected return (cumulative reward) of taking action 
$a$ in state $s$ and then following a given policy thereafter, defined by $q_{\pi}(s, a) = \mathbb{E_{\pi}} \left[ \sum_{t=0}^\infty \gamma^t r_{t+1} \;\middle|\; s_0 = s, a_0 = a \right]
$.} given by
\begin{align}\label{eq:Bolz_distribution}
\pi_\text{LM}(a \mid s) = \frac{\exp\bigl(q(s,a)\bigr)}{\sum_{a'} \exp\bigl(q(s,a')\bigr)}.
\end{align}
In particular, a language model's pre-softmax logit can be viewed as the Q-value function due to the shared formula. With the Bellman optimality equation~\cite{bellman1952theory}, a step-wise reward function can be induced by
\begin{align} \label{eq:reward_funcion}
r(s_t,a_t) = q(s_t,a_t) - \max_{a' \in \mathcal{A}} q(s_{t+1}, a').
\end{align}
Then, the goal of RL for text generation KD is to optimize the student's policy, denoted by $\pi_\theta$, to maximize the expected cumulative reward:
\begin{align}
J(\theta) = \mathbb{E}_{\pi_\theta} \Biggl[ \sum_{t=1}^{T} r(s_t,a_t) \Biggr].
\end{align}
The REINFORCE algorithm \cite{williams1992simple} is a policy gradient method, which is widely used for RL in NLP~\cite{hao2022teacher,li-etal-2024-llmr}. 
\begin{align}\label{eq:pg}
\nabla_\theta J(\theta) &= \mathbb{E}_{\pi_\theta} \left[ \sum_{t=1}^{T} G_t \nabla_\theta \log \pi_\theta(a_t \mid s_t)  \right]
\end{align}
where $G_t = \sum_{i=t}^{T} r(s_{i}, a_{i})$ is a cumulative reward (i.e., return) from step $t$, and the expectation is approximated by Monte Carlo samples from the distribution $\pi_\theta$. 

\subsection{Our \name{} Method}
\label{sec:xxx_method}
In this work, we address RL-based KD and propose to refine the learning signal $G_t$ in Eqn.~\eqref{eq:pg} by extending the one-step reward induction to $K$ steps, which alleviates the high variance issue of RL. The key idea is to apply the Bellman optimality equation for multiple steps, therefore directly connecting the Q-values at the current state with those of a future state.

We begin by considering the sum of rewards in Eqn.~\eqref{eq:reward_funcion} over $K$ consecutive steps starting from step $t$, denoted by $G_{t:t+K}$
\begin{align}
\label{eq:k_step_BO_simplify}
G_{t:t+K}\!:=&\sum_{i=0}^{K-1} r(s_{t+i}, a_{t+i}) \notag \\ 
=& \sum_{i=0}^{K-1}\Big[q(s_{t+i}, a_{t+i}) -\max_{a'\in \mathcal A}q(s_{t+i+1},a')\Big] \notag \\ 
=& q(s_{t}, a_{t}) - \max_{a' \in \mathcal {A}} q(s_{t+K},a'), 
\end{align}
where Eqn.~\eqref{eq:k_step_BO_simplify} assumes that an optimal action $a_{t+i+1} = \arg\max_{a' \in \mathcal{A}}\, q\bigl(s_{t+i+1}, a'\bigr)$ is taken. However, a student's policy may not be optimal; therefore,  Eqn.~\eqref{eq:k_step_BO_simplify} becomes an approximation, denoted by $\hat G_{t:t+K}$,:
\begin{align}
\hat G_{t:t+K} = q(s_{t}, a_{t}) - \max_{a' \in \mathcal {A}} q(\hat{s}_{t+K},a') 
\end{align}
where $\hat{s}_{t+K}$ is the state at the $(t+K)$th step after following the student's policy. This is a reasonable approximation because, in KD, a student is usually pretrained in a meaningful way ~\cite{turc2019wellread,lee2023dwt,kim2024promptkd} and the approximation will be more accurate as the optimization proceeds.

Building upon the $K$-step reward formulation, we can obtain an approximate return $\hat G_t$ by considering intervals of $K$ steps, i.e., $\hat G_{t:t+K}, \hat G_{t+K:t+2K},\cdots$. Formally, we have
\begin{equation}\label{eq:xxx_return}
\resizebox{0.95\linewidth}{!}{$
\begin{aligned}
\hat{G}_t
=& \sum_{i=0}^{\left\lfloor \frac{T-t+1}{K} \right\rfloor}\hat{G}_{t+iK:t+(i+1)K} \\ 
\!=& \!\!\sum_{i=0}^{\left\lfloor \frac{T-t+1}{K} \right\rfloor}
  \Bigl[
    q(s_{t+iK}, a_{t+iK}) 
    \!-\!\max_{a' \in \mathcal{A}}\, q(\hat{s}_{t+(i+1)K}, a')
  \Bigr].
\end{aligned}
$}
\end{equation}
which will be used in our RL-based generation KD.

Overall, the student’s policy is used to sample a sequence of actions (i.e., output words). Then, the sequence is fed to the teacher model, which evaluates the sequence by Eqn.~\eqref{eq:xxx_return}. Finally, we follow the policy gradient formula, but use the approximate return for the update:
\begin{align}\label{eq:pg1}
\nabla_\theta J(\theta) &\approx \mathbb{E}_{\pi_\theta} \left[ \sum_{t=1}^{T}\hat G_t \nabla_\theta \log \pi_\theta(a_t \mid s_t)  \right]
\end{align}
where $\hat G_t$ is given by Eqn.~\eqref{eq:xxx_return}. The process is shown in Algorithm \ref{alg:rl_kstep}. 
\input{latex/algorithm}

\subsection{Bias and Variance Analysis}
\label{sec:bias_variance_analysis}

Although the REINFORCE algorithm \cite{williams1992simple} estimates gradients in an unbiased way, it is known to be noisy and prone to high variance in the gradient estimation, which may lead to instability in learning \cite{greensmith2004variance, mnih2016asynchronous, bjorckhigh}.

A standard method to mitigate this issue is to subtract a \emph{baseline} term $b_t$ from the actual return:
\begin{align} \label{eq:baseline}
\hat{G}_t = G_t - b_t.
\end{align}
For example, the average return over a batch ~\cite{rosenberg2021variance} is commonly used as the baseline term to stabilize the REINFORCE algorithm.

Our \name{} approach is a variant of REINFORCE with baseline. This can be seen by examining the difference between the actual return  $G_t$ and our approximate return $\hat G_t$. In our KD application, the actual return $G_t$ is given by
accumulating the reward defined in Eqn.~\eqref{eq:reward_funcion}. In other words, we have
\begin{align}\label{eq:sample_return}
G_{t} = \sum_{i=0}^{T} \Bigl( q(s_{t+i}, a_{t+i}) - \max_{a' \in \mathcal{A}} q(s_{t+i+1}, a') \Bigr).
\end{align}
Combining Eqns.~\eqref{eq:xxx_return}, \eqref{eq:baseline}, and \eqref{eq:sample_return}, we can interpret our approximate return $\hat G_t$ as introducing a baseline term with the following form
\begin{align}\label{eq:xxx_baseline}
\resizebox{1.03\linewidth}{!}{$
\begin{aligned}
b_t 
= \sum_{\substack{i=0 \\ i \not\equiv 0 \,(\mathrm{mod}\, K)}}^{T-1} \Bigl[
      q(s_{t+Ki+1}, a_{t+Ki+1}) -
      \max_{a' \in \mathcal{A}} q(s_{t+Ki+1}, a') 
    \Bigr].
\end{aligned}
$}
\end{align}

Unlike conventional, policy-independent baselines \cite{sutton2018reinforcement, rosenberg2021variance}, our baseline depends on the selected actions and thus introduces bias into the expected return estimation. However, our approach can alleviate the high variance issue of REINFORCE with mild assumptions, as shown by the following theorem.

\begin{theorem}[Variance Reduction via $K$-Step Return]
\label{thm:variance_reduction}
Let $G_t$ be the actual return and $\hat{G}_t$ be the $K$-step approximate return for some sequences sampled from the student policy $\pi$. Assuming that the state--action--reward tuples $(s_t, a_t, r_t)$ are iid drawn at different steps, we have:
\begin{align} \label{eq:variance}
\mathrm{Var}[\hat{G}_t] \le \mathrm{Var}[G_t].
\end{align}

\end{theorem}
\begin{proof}
See Appendix~\ref{apd:proof_theorem1}.
\end{proof}

The iid assumption is reasonable and widely adopted in theoretical RL research \cite{kearns2000bias,bhandari2018finite,xureanalysis}, because in many environments the dependencies decay rapidly and correlation is further weakened when a large batch of samples are considered.

Overall, Theorem~\ref{thm:variance_reduction}, along with the derivations in Appendix~\ref{apd:proof_theorem1}, indicates that our \name{} alleviates variance at a power rate as $K$ increases. Although this method introduces a bias term in the gradient estimation, the bias is effectively mitigated: it diminishes for smaller values of $K$ and converges to zero if the student policy becomes optimal. Detailed bias analysis is given in Appendix~\ref{apd:bias_analysis}. Such bias-variance trade-off is widely applied in existing RL literature, as seen in Temporal Difference (TD) learning \cite{sutton1988learning}, Actor--Critic algorithms~\cite{konda1999actor, mnih2016asynchronous}, and Deep Q-Network~\cite[DQN;][]{mnih2015human}.

\section{Experiments}

\input{latex/vol_reward}
In this section, we present the empirical evaluation and analysis of our proposed \name. We begin by describing the datasets, baseline methods, and implementation 
details, followed by the main results and detailed analyses.

\subsection{Settings}

\paragraph{Tasks, Datasets, and Metrics.} We evaluate our approach on various text generation tasks that are frequently considered in previous literature~\cite{maruf2018contextual, magister2023teaching, wen-etal-2023-f, touvron2023llama, biderman2024lora, wang2024self}.
\begin{compactitem}[$\bullet$]
\item \textbf{XSum Summarization.} The Extreme Summarization (XSum) is a challenging dataset for text summarization introduced by \citet{narayan-etal-2018-xsum}, where the summaries are highly abstractive as they emphasize key ideas with novel wordings. The dataset consists of approximately 226,000 BBC articles paired with single-sentence summaries. We employ 
ROUGE~\cite{lin-2004-rouge} as the primary metric, which is a common practice in 
summarization \cite{ravaut2024context,van2024adapted,agarwal2025many}.

\item \textbf{Europarl EN--NL Translation.} Europarl \cite{koehn2005europarl} is a high-quality, multilingual parallel corpus extracted from European Parliament proceedings. Its content is professionally produced and carefully aligned, guaranteeing data that is both reliable and clean. We choose English-to-Dutch, a relatively low-resource translation direction, to facilitate our distillation experiments. We report the BLEU score \cite{papineni2002bleu}, character-level F score \cite[chrF,][]{popovic-2015-chrf}, and translation edit rate \cite[TER,][]{snover2006study}, 
following the standard evaluation in machine translation 
\cite{barrault-etal-2019-findings,hrabal2024cuni}.

\item \textbf{GSM8K Reasoning.} Grade School Math 8K~\cite[GSM8K,][]{cobbe2021training} is a popular dataset consisting of around 8,000 grade school-level math problems with detailed step-by-step solutions. It is designed to evaluate a model's abilities in mathematical reasoning and multi-step problem-solving. The standard evaluation metric for GSM8K is solution accuracy \cite{wang2024self, setlur2025rl}, which is adopted in our experiments.
\end{compactitem}
We employ the standard training, validation, and test splits for XSUM~\cite{narayan-etal-2018-xsum} and Europarl~\cite{koehn2005europarl}. For GSM8K, the standard split comprises only training and test sets~\cite{cobbe2021training}. We adopt the open-source validation split provided by \citet{wang2024self}.

\paragraph{Implementation Details.}\label{subsec:implementation_detals} 
In our KD, the teacher is the 3B-parameter FLAN-T5-XL model~\cite{chung2024scaling}, which shares the same architecture as prior work \cite{li-etal-2024-llmr}. For the summarization task, we directly prompt FLAN-T5-XL as it has already been instruction-finetuned for summarization. On the other tasks, FLAN-T5-XL yields subpar performance if prompted directly; we finetune the model as the teacher, which is common in KD research \cite{de2024hybrid,setiawan2024accurate,ye2025best}. 

Our student uses the 250M-parameter T5-base model \citet{t52020}, which is consistent with the configuration in \citet{wen-etal-2023-f} and \citet{li-etal-2024-llmr}. 

Also following previous KD studies \cite{wen-etal-2023-f,li-etal-2024-llmr}, we perform pre-distillation, where the student is pretrained by the cross-entropy loss based on the teacher's outputs. This ensures a meaningful initialization of the student model and enables effective exploration for reinforcement learning. Notice that text generation has a much larger state--action space than a typical RL environment such as Atari games~\cite{mnih2015human}. The student performs greedy action selection when generating a sequence. Our return induction builds upon $K$-step Bellman optimality equations, and the hyperparameter $K$ is critical in our framework. We report performance for $K \in \{2, 4, 8, 16\}$ in our experiments.

\paragraph{Competing Methods.}
We compare our \name{} against both divergence-based and RL-based text generation KD:
\begin{compactitem}
    \item \textbf{SeqKD} \cite{kim-rush-2016-sequence}. This is a classic method where the student maximizes likelihood of teacher-generated sequences.
    
    \item \textbf{KL Distillation} \cite{hinton2015distillingknowledgeneuralnetwork}. It minimizes the Kullback--Leibler (KL) divergence between student and teacher distributions. Notice that SeqKD is a hard version of KL distillation.
    
    \item \textbf{JS Distillation} \cite{wen-etal-2023-f}. Jensen--Shannon (JS) divergence is a symmetric 
    divergence that overcomes the over-smoothing problem of KL divergence~\cite{wei2019neural,wen-etal-2023-f}.
    
    \item \textbf{TVD Distillation} \cite{wen-etal-2023-f}. The Total Variation Distance (TVD) is another symmetric divergence and is shown to outperform other methods~\cite{wen-etal-2023-f}. Such a method is also explored in~\citet{agarwal2024onpolicy} with a tunable ratio between the two terms of TVD.
    
    \item \textbf{LLMR} \cite{li-etal-2024-llmr}. In this method, a reward function is induced from a teacher language model by one-step Bellman optimality~\cite{hao2022teacher}. Then, the student model is trained by RL towards the induced reward.
\end{compactitem}
Since our approach reduces the variance of RL, we consider alternative variance reduction techniques under the LLMR framework:
\begin{compactitem}
    \item \textbf{LLMR + Mean Baseline.} Using the average reward in a batch as a baseline is commonly used for stabilizing RL training~\cite{sutton2018reinforcement}.
    
    \item \textbf{LLMR + Min-Variance Baseline.} This is an advanced variant that is shown to be theoretically optimal when the baseline is derived from batch data \cite{rosenberg2021variance}.
\end{compactitem}

For a fair comparison, we apply the same settings in Section~\ref{subsec:implementation_detals} (when applicable) to the competing methods as we do to our approach.
Specifically, all methods adopt pre-distillation to ensure a meaningful student initialization, and all RL methods use the same action selection procedure.

\subsection{Main Results}
\input{latex/main_results}

As mentioned in Section~\ref{sec:xxx_method}, the primary advantage of our \name{} is its enhanced RL optimization compared with classic REINFORCE. In this part, we will first show that our approach indeed achieves a higher return (cumulative reward) in RL. Then, we will show that our approach leads to improved performance in NLP tasks.

\paragraph{Return in RL.} The goal of RL is to learn a policy maximizing the cumulative reward, also known as the return. Therefore, we may use it to evaluate the outcome of RL training. 

Figure~\ref{fig:vol_plot} shows the return score that is defined in Eqn.~\eqref{eq:sample_return}, where the return is averaged over different test samples, using various RL methods in the three NLP tasks. As seen, our \name{} consistently achieves a higher average return than competing approaches across all the tasks. This indicates that our \name{} learns a superior policy in terms of the return, which is precisely the RL optimization objective.

In addition, we observe that an increased $K$ may not necessarily improve the return. This is because our \name{} introduces bias despite its reduced variance (Section~\ref{sec:bias_variance_analysis}). Therefore, a trade-off should be sought when choosing the $K$ value.

\paragraph{NLP Task Performance.}
Table~\ref{tab:main-result} presents the results of our approach in terms of text generation performance.

We first examine the performance of directly prompting the teacher and the non-distilled student model in a zero-shot manner, offering empirical lower and upper bounds for the KD process. Note that the bounds are not theoretically guaranteed; instead, KD is empirically expected to improve the student's performance but may still underperform the teacher, especially when the student is small. In our setup, the student is a T5-base model, which does not yield reasonable performance by itself when prompted directly.

We then consider divergence-based distillation methods, including SeqKD and KL/JS/TVD distillations. 
As seen from the table, symmetric methods (JS, TVD)---which involve both exploitation of teacher predictions and exploration based on student predictions---tend to surpass asymmetric methods (SeqKD, KL), where the student follows teacher predictions without any exploration. The results are consistent with previous findings \cite{wen-etal-2023-f,agarwal2024onpolicy}.

Next, we evaluate LLMR~\cite{li-etal-2024-llmr}, a text generation KD approach using REINFORCE. Results show that LLMR provides certain performance gain over non-RL KD methods, which is likely stemmed from the student’s self-exploration, aligning with the observations in \citet{li-etal-2024-llmr} and other recent RL-based text generation research \cite{ouyang2022training, liu2024deepseek, deepseekai2025deepseekr1incentivizingreasoningcapability}.

To mitigate the high variance of REINFORCE in LLMR, we incorporate classic RL baseline terms (mean baseline and min-variance baseline) that are estimated from batch data. However, these methods are not effective in our scenario, as text generation has a very large state--action space, which makes the generated outputs in a batch less representative and the baseline term less useful.

By contrast, our \name{} employs a novel baseline formulation that largely reduces the variance of RL (Theorem~\ref{thm:variance_reduction}) and improves RL optimization (Figure~\ref{fig:vol_plot}). Consequently, it delivers a noteworthy add-on performance gain on top of LLMR across three text generation tasks. 

In the experiment, we also observe that a moderate $K$ between 2 to 8 leads to the highest NLP performance, which is consistent with the return analysis in Figure~\ref{fig:vol_plot}. It is also noticed that RL return and NLP performance are not perfectly correlated, as the induced reward may not fully reflect the task metric such as BLEU and ROUGE scores, which is also known as reward hacking ~\cite{amodei2016concrete,hao2022teacher,ouyang2022training}.

\textbf{Summary.} 
Our main results show that the proposed \name{} (with a moderate $K$) improves RL optimization, which is generally translated to higher performance in various NLP tasks.

\subsection{In-Depth Analyses}\label{sec:analysis}

\paragraph{Variance and bias analysis.}
\input{latex/variance}
\input{latex/plot_model_size} As shown by the theoretical analysis in Section~\ref{sec:bias_variance_analysis}, our approach provides a bias--variance trade-off by largely reducing the variance, although introducing a bias term. We empirically verify them in this analysis.  

Figure~\ref{fig:variance} shows the variance of the $K$-step return, where we sample $32$ outputs for a given input and use Eqn.~\eqref{eq:xxx_variance} to estimate the variance of return; the variance is further averaged over 10K input samples. For the bias, we use Eqn.~\eqref{eq:return_bias} for empirical estimation, and the results are shown in Figure~\ref{fig:bias}. We choose the value of $K$ from $\{1, 2, 4, 8, 16\}$ to see the trends. Note that $K=1$ corresponds to the competing approach LLMR~\cite{li-etal-2024-llmr}. In addition, we examine the impact of the initial student policy by considering students with various KL divergence levels from the teacher policy: a smaller KL divergence indicates that the student and teacher are more resemblant. 

We observe that the variance decreases drastically as $K$ increases, while the bias term increases steadily. The observations align with our theoretical analysis in Section~\ref{sec:bias_variance_analysis} and Appendix~\ref{apd:bias_analysis}, suggesting the need for seeking a moderate $K$ value to balance bias and variance.\footnote{Our bias--variance trade-off is different from that in a regression analysis~\cite{ hastie2009elements, vapnik2013nature}, where the total squared error is the sum of variance and squared bias, plus an irreducible noise. By contrast, the variance of return affects the smoothness of RL training, while bias affects the optimum quality (if converging); their total effect is not given by a simple addition.}

We also observe that when the student policy is initialized closer to the teacher policy (i.e., a smaller KL divergence), our \name{} generally demonstrates lower bias and variance. The bias reduction is predicted by our theoretical analysis in Appendix~\ref{apd:bias_analysis}, whereas the variance reduction is an empirical observation. Overall, the results demonstrate that pre-distillation is important to RL training for text generation, which is consistent with previous work \cite{ouyang2022training, li-etal-2024-llmr, deepseekai2025deepseekr1incentivizingreasoningcapability}.

\paragraph{Model Size.} We analyze RL-based KD approaches with different student sizes. Figure~\ref{fig:plot_model_size} presents the learning curves for student models initialized from T5-small (77M parameters), T5-base (250M parameters), and T5-large (800M parameters) using our $K$-step approach and the competing LLMR approach. 
As seen from the learning curves in Figure~\ref{fig:plot_model_size}, the LLMR approach exhibits notable instability during RL training as the model size increases, especially when scaling to T5-large. Such a phenomenon is also reported in the RL literature: a large network is prone to overfit the limited sampled outputs, consequently leading to unstable performance on test data~\cite{henderson2018deep,cobbe2019quantifying}.

On the contrary, our \name{} largely alleviates this issue by reducing the variance, which stabilizes the learning curves. As a result, our method achieves smoother training and higher performance with all model sizes, compared with the LLMR approach.

\paragraph{LLM Evaluation.} We conduct an LLM evaluation as a surrogate of human evaluation, as classic NLP metrics (such as ROUGE and BLEU) may not fully reflect the quality of generated text. Specifically, we prompt the \texttt{Qwen2.5-72B-Instruct} LLM~\cite{qwen2025qwen25technicalreport}  to conduct a pairwise evaluation of system outputs, against the commonly used KL distillation. We select TVD, LLMR, and our \name{} from Table~\ref{tab:main-result} as the competitors, as pairwise evaluation is expensive. Our LLM evaluation considers multiple criteria, including overall quality, informativeness, and coherence. For each comparison, we query the LLM four times by swapping the two candidates and their IDs (namely, A and B), as LLM is prone to ID bias~\cite{zheng2023large} and positional bias~\cite{shen2023large}, following ~\citet{huang-etal-2024-decoding}'s setup.
The detailed prompts are presented in Appendix~\ref{apd:llm_eval_prompt}.

Table~\ref{tab:human} shows the results of the LLM evaluation. We observe that our \name{} achieves the best winning rate in terms of all criteria (overall quality, informativeness, and coherence) on both datasets. These compelling results are consistent with the traditional task metrics in Table~\ref{tab:main-result} and further demonstrate the effectiveness of our \name{}.

\begin{table}[!t]
\centering
\resizebox{\linewidth}{!}{
\begin{tabular}{llcccc}
\toprule
Dataset&Method & Overall& Informativeness &Coherence&\\\midrule

\multirow{3}{*}{XSum} &TVD&67.50\%&68.15\%&65.90\%\\
&LLMR&69.95\%&70.55\%&66.30\%\\
&\name{}&\textbf{73.50\%}&\textbf{73.90\%} &\textbf{70.40\%}\\
\midrule
\multirow{3}{*}{Europarl} &TVD&53.80\%&54.15\%&54.85\%\\
&LLMR&56.45\%&55.85\%&56.30\%\\
&\name{}&\textbf{58.85\%}&\textbf{57.95\%} &\textbf{58.45\%}\\
\bottomrule
\end{tabular}}
\caption{LLM-based evaluation on the summarization and translation tasks. Europarl refers to the Europarl EN--NL dataset. We show the winning rates of each method over the KL distillation baseline in terms of overall quality, informativeness, and coherence.}
\label{tab:human}
\end{table}

\section{Related Work}

\paragraph{Knowledge Distillation.} The foundation of KD is laid by \citet{buciluǎ2006model}, who performs KD by aligning the logits of the student with those of a teacher through squared error minimization. This framework is extended by \citet{hinton2015distillingknowledgeneuralnetwork}, who propose to use KL divergence to match the output probability distributions of the teacher and student. \citet{kim-rush-2016-sequence} extend KD to the sequence level for auto-regressive models, and \citet{wen-etal-2023-f} further propose a general framework of $f$-divergence minimization to mitigate the mode averaging and collapsing issues. These divergence-based KD approaches heavily rely on imitation of the teacher’s predictions. 

\paragraph{Reinforcement Learning for Text Generation.} Reinforcement learning (RL) offers a framework that enables a language model to explore during training. A key challenge in RL-based text generation lies in designing reward signals. Early efforts by \citet{wu2018study} employ task-specific metrics (e.g., BLEU for machine translation) as rewards, while \citet{ouyang2022training} leverage human preference data to train discriminative reward models. However, such methods require human engineering or human annotation. 

\paragraph{Bridging RL and text generation KD.} Recent work has sought to combine RL and KD by deriving rewards from teacher models. \citet{hao2022teacher} derive a step-wise reward function via Bellman 
Optimality equation, which alleviates the sparse reward issue commonly existing in other RL text generation scenarios~\cite{wu2018study,ouyang2022training}. Building on this, \citet{li-etal-2024-llmr} extend this approach to KD settings, where they induce a reward function from a large language model (serves as a teacher) and train a 
student model to maximize the teacher-induced cumulative reward. However, RL is known to suffer from high variance, and our paper proposes \name{} that largely reduces the variance of RL training. 

\paragraph{Variance Reduction in RL.}  REINFORCE with baseline \cite{sutton2018reinforcement, rosenberg2021variance} mitigates the high variance issue by subtracting a baseline term derived from batch data. Actor--Critic methods \cite{konda1999actor, mnih2016asynchronous} address this by learning a value function (critic) as the baseline term, but the inaccurate value estimates from the critic can lead to harmful updates in the actor’s policy, while a poor decision by the actor can adversely affect the critic's learning.
This often results in divergence of RL training~\cite{bhatnagar2007incremental,fujimoto2018addressing,parisi2019td}. 
Recent RL work for large language models avoids learning a critic as the baseline term \cite{deepseekai2025deepseekr1incentivizingreasoningcapability}. Our \name{} exploits the mathematical structure of LM-induced rewards to derive a principled baseline for variance reduction, without learning an auxiliary neural network like a critic.

Another line of studies develops conservative policy optimization techniques like TRPO \cite{pmlr-v37-schulman15} and PPO \cite{schulman2017proximal}, which constrain policy updates to prevent instability. Our work of estimating $K$-step return is compatible with this line of research. This goes beyond the scope of our paper, and can be explored in future work.

\paragraph{$N$-Step Bootstrapping.}  
In value-based RL, the state-value function can be estimated either from one-step Temporal Difference (TD) bootstrapping ($N=1$) or from the full Monte Carlo return ($N\!\to\!\infty$). Classic $N$-step TD interpolates between these two extremes.  
Although this appears related to our method, the $N$-step formulation is fundamentally different from our \name. First, the symbols $N$ and $K$ encode distinct mathematical concepts: $N$ specifies the rollout length of the TD target, whereas $K$ denotes the depth of the inverse Bellman expansion that defines our variance-reduction baseline. Second, $N$-step TD requires learning an additional parameterized state-value function, which is impractical in our setting. Thus, algorithms built on $N$-step bootstrapping, including Eligibility Trace \citep{sutton1988learning} and its modern variants \citep{schulman2015high}, lie outside the scope of this work.

\section{Conclusion}
In this paper, we introduce \name{}, a $K$-step return induction framework for reinforcement learning based knowledge distillation in the text generation domain. Compared with conventional RL methods, our approach effectively reduces gradient variance, shown by both theoretical and empirical analyses. Extensive experiments across diverse text generation tasks verify that our approach improves RL training and boosts NLP task performance.

\section*{Limitations}
While our work demonstrates both theoretical depth and empirical effectiveness, it is not without limitations. 
First, our RL-based knowledge distillation optimizes an induced reward function, which may not fully align with the NLP task~\cite{ouyang2022training, pan2022effects, gao2023scaling}. Nevertheless, our experiments support the claim that a better RL optimization generally leads to improved NLP metrics, as shown in Table~\ref{tab:main-result}.
Also, traditional NLP metrics (such as ROUGE and BLEU scores) may not fully reflect human judgment. Therefore, we have also conducted LLM evaluation as a surrogate of human studies~\cite{chiang2023can,liu2023g, lin2023llm}, during which we have carefully eliminated the bias of LLMs~\cite{zheng2023large,shen2023large}.

Furthermore, RL for LLMs often suffers the challenge of \textit{reward hacking}, where policy optimization improves the proxy reward model while actually degrading performance
with respect to the true human preference, due to the imperfection of the proxy \cite{ouyang2022training,gao2023scaling,fan2026reward}. In the context of RL-based KD, the reward function is induced from the teacher model. As the teacher model’s distribution approaches optimal, its induced rewards increasingly align with the underlying true reward \cite{hao2022teacher}. This suggests that reward hacking can be effectively mitigated by utilizing a more capable teacher model. Our \name{} method smooths the learning signal, ensuring the student model more reliably follows the teacher's guidance, thus translating improved RL optimization into better performance across various NLP tasks.

\section*{Acknowledgments}
We thank the reviewers and area chairs for their
efforts. The research is supported in part by the
Natural Sciences and Engineering Research Council of Canada (NSERC), the Amii Fellow Program, the Canada CIFAR AI Chair Program, a donation
from DeepMind, and the Digital Research Alliance
of Canada (alliancecan.ca).

\bibliography{custom}

\appendix

\onecolumn

\section{Proof of Theorem~\ref{thm:variance_reduction}}\label{apd:proof_theorem1}
Using $K$-step returns as a learning signal to learn a student policy $\pi$ guarantees reduced variance in return estimation compared to the full trajectory return, i.e., $\mathrm{Var}[\hat{G}_t] \leq \mathrm{Var}[G_t]$. (Detailed in Theorem~\ref{thm:variance_reduction}).
\begin{proof}
We denote the variance of $q(s,a)$ and $\max_{a' \in \mathcal{A}}\, q(s,a')$ as:
\begin{align}
  \sigma_{\mathcal S,\mathcal A}^2 &=\mathrm{Var}_{s,a}\bigl[q(s,a)\bigr], \\
 \sigma_{\mathcal S}^2 &= \mathrm{Var}_s\Bigl[\max_{a' \in \mathcal{A}}\, q(s,a')\Bigr] .
\end{align}

We first decompose the variance of the actual return $G_t$: 
\begin{align}
\mathrm{Var}[G_t] &= \mathrm{Var}\bigg[\,\sum_{i=0}^{T-t} r_{t+i}\,\bigg] & \text{[definition of $G_t$]}  \\
&= \sum_{i=0}^{T-t} \mathrm{Var}\Bigl[q(s_{t+i}, a_{t+i}) - \max_{a' \in \mathcal{A}}\, q(s_{{t+i}+1}, a')\Bigr] &\text{[iid assumption]}\\
&= \sum_{i=0}^{T-t}\Bigr( \mathrm{Var}\bigl[q(s_{t+i}, a_{t+i})\bigr] + \mathrm{Var}\bigl[\max_{a' \in \mathcal{A}}\, q(s_{{t+i}+1}, a')\bigr] \Bigr) & \text{[iid assumption]} \\
&= \sum_{i=0}^{T-t}\bigl(\sigma_{\mathcal S,\mathcal A}^2 + \sigma_{\mathcal S}^2 \bigr) \\
&= (T-t+1)\bigl(\sigma_{\mathcal S,\mathcal A}^2 + \sigma_{\mathcal S}^2 \bigr) \label{eq:var_g}.
\end{align}

Next, we decompose the variance of our $K$-step approximate return $\hat{G}_t$:
\begin{align}
\mathrm{Var}[\hat{G}_t] &= \mathrm{Var}\Biggl[\sum_{i=0}^{\left\lfloor \frac{T-t}{k} \right\rfloor}
  \Bigl( q(s_{t+iK}, a_{t+iK}) - \max_{a' \in \mathcal{A}}\, q(s_{t+(i+1)k}, a') \Bigr)\Biggr] \quad &\text{[by Eqn.~\eqref{eq:xxx_return}]} \label{eq:xxx_variance}\\
&= \sum_{i=0}^{\left\lfloor \frac{T-t}{k} \right\rfloor} \mathrm{Var}\Bigl[q(s_{t+iK}, a_{t+iK}) - \max_{a' \in \mathcal{A}}\, q(s_{t+(i+1)k}, a')\Bigr] &\text{[iid assumption]}\\
&= \sum_{i=0}^{\left\lfloor \frac{T-t}{k} \right\rfloor} \Bigl( \mathrm{Var}\bigl[q(s_{t+iK}, a_{t+iK})\bigr] + \mathrm{Var}\bigl[\max_{a' \in \mathcal{A}}\, q(s_{t+(i+1)k}, a')\bigr] \Bigr) &\text{[iid assumption]}\\
&= \sum_{i=0}^{\left\lfloor \frac{T-t}{k} \right\rfloor} \bigl(\sigma_{\mathcal S,\mathcal A}^2 + \sigma_{\mathcal S}^2 \bigr) \\
&= \bigl(\left\lfloor \frac{T-t}{k} \right\rfloor +1 \bigr)\bigl(\sigma_{\mathcal S,\mathcal A}^2 + \sigma_{\mathcal S}^2 \bigr). \label{eq:var_g_hat} 
\end{align}

Comparing Eqns.~\eqref{eq:var_g} and~\eqref{eq:var_g_hat}, we immediately have $\mathrm{Var}[\hat{G}_t] \le \mathrm{Var}[G_t]$, completing the proof.
\end{proof}

\section{Bias Analysis} \label{apd:bias_analysis}

In this appendix, we analyze the bias introduced by using the $K$-step return $\hat{G}_t$ in place of the actual return $G_t$. Recall that they differ by a baseline term shown in Eqns.~(\ref{eq:baseline}) and~(\ref{eq:xxx_baseline}) in the main text, and this discrepancy introduces bias in the return estimation:

\begin{align} \label{eq:return_bias}
\text{bias of return} = \mathbb{E}_{\pi_\theta}\Bigl[(\hat{G}_t-G_t) \Bigr] = \mathbb{E}_{\pi_\theta}\Biggl[\sum_{\substack{i=0 \\ i \not\equiv 0 \,(\mathrm{mod}\, K)}}^{T-1} \Bigl[q(s_{t+Ki+1}, a_{t+Ki+1})  - \max_{a' \in \mathcal{A}} q(s_{t+Ki+1}, a') \Bigr]\Biggr].
\end{align}

The gradient estimation is thus biased by:
\begin{align} \label{eq:graident_bias}
\text{bias of gradient} = \mathbb{E}_{\pi_\theta}\Bigl[(\hat{G}_t-G_t) \nabla_\theta \log \pi_\theta(a_t \mid s_t)\Bigr] = \mathbb{E}_{\pi_\theta}\Bigl[-b_t \nabla_\theta \log \pi_\theta(a_t \mid s_t)\Bigr].
\end{align}

We show below that a smaller value of $K$ reduces bias, providing a bias--variance tradeoff for REINFORCE. Further, we will show that the bias converges to zero if the student policy becomes optimal, assuming all Q-values are distinct.

\paragraph{Bias Reduction with Smaller $K$.}

The baseline term defined in Eqn. (\ref{eq:xxx_baseline}) in the main text is given by
\begin{align} \label{xxx_baseline2}
b_t = \sum_{\substack{i=0 \\ i \not\equiv 0 \,(\mathrm{mod}\, K)}}^{T-1} \Bigl[q(s_{t+Ki+1}, a_{t+Ki+1})  - \max_{a' \in \mathcal{A}} q(s_{t+Ki+1}, a') \Bigr].
\end{align}
Since
\begin{align}
  q(s_{t+Ki+1}, a_{t+Ki+1}) - \max_{a' \in \mathcal{A}} q(s_{t+Ki+1}, a') \le 0,
\end{align}
a smaller $K$ reduces the number of terms in the summation. This decreases $|b_t|$, which in turn decreases the magnitude of the gradient bias in Eqn.~(\ref{eq:graident_bias}).

\paragraph{Bias Convergence to Zero.} Suppose the student policy is optimal, i.e., greedy with respect to the teacher's Q-value function $q(s,a)$, given by 
\begin{align} \label{eq:greedy_action}
a_{t+i} = \arg \max_{a' \in \mathcal{A}}\, q(s_{t+i}, a').
\end{align}
It is easy to see from Eqn.~(\ref{xxx_baseline2}) that $b_t=0$, implying that 
\begin{align}
\mathbb{E}_{\pi_\theta}\Bigl[b_t \nabla_\theta \log \pi_\theta(a_t \mid s_t)\Bigr] = 0.
\end{align}
Suppose the Q-values for different actions are distinct (in which case $\operatorname{argmax}$ is unambiguous), the result further suggests that the bias term would converge to zero, if the student policy is closer to optimal during training.

\twocolumn
\input{latex/dataset_stats}
\section{Details of Experimental Settings}\label{sec:exp_detail}

\textbf{Computing Infrastructure.} Experiments were conducted on a Linux server equipped with an AMD EPYC 7313 CPU (32 $\cdot$ GB RAM) and an NVIDIA RTX A6000 GPU (48 $\cdot$ GB VRAM). The system uses NVIDIA driver v560.28.03 and CUDA Toolkit 12.6 (as reported by \texttt{nvidia-smi}). Software and dependency versions are listed in the \texttt{requirements.txt} file of our anonymous GitHub repository: \url{https://github.com/MANGA-UOFA/KETCHUP}.

\textbf{Hyperparameter Settings.} For our RL-based distillation experiments, we follow the configuration of \citeauthor{li-etal-2024-llmr}~(\citeyear{li-etal-2024-llmr}), employing the AdamW optimizer~\citep{loshchilov2018decoupled} with default parameters (\(\beta_1=0.9,\,\beta_2=0.999\)). All other hyperparameters—batch size, gradient accumulation steps, reward clipping range, dropout rate, warmup steps, and learning rate—are identical to those in \citep{li-etal-2024-llmr}. Since our datasets differ, we adjust the maximum input and output lengths for each text-generation task according to the recommendations of \citep{wen-etal-2023-f, wang2024self}. Detailed settings are provided in Table~\ref{tab:hyp_rl}.

For divergence-based KD competing approaches, we adopt the hyperparameter configurations from \citep{wen-etal-2023-f}. The specific values for each parameter are summarized in Table~\ref{tab:hyp_div}.

\textbf{Statistical Analysis.}  
To quantify the variance of each approach we repeat every training procedure $N=5$ times, using distinct random seeds drawn uniformly from $([1,100\,000])$: $\{19083,\,34007,\,84122,\,310,\,55080\}$.  
For every run, we select the checkpoint that achieves the best validation performance and report its corresponding test‑set score.  
We adopt the standard train/validation/test splits for \textsc{XSum} and \textsc{Europarl}.  
Because the official \textsc{GSM8K} release lacks a validation split, we use the public split of \citet{wang2024self}.  Statistical significance between our method \name and each baseline is assessed with a paired two‑sided \emph{t}‑test over the five seeds; $p<0.035$ indicating the differences are deemed significant.

\begin{table}[!b]
    \centering
    \resizebox{1\columnwidth}{!}{
    \begin{tabularx}{0.65\textwidth}{lX}
        \toprule
        \textbf{Hyperparameter} & \textbf{Value} \\
        \midrule
        Training Epochs & 3 \\
        Train Batch size & 8 \\
        Eval Batch size & 32 \\
        Optimizer & AdamW \\
        Grad Accumulation Steps & 32 \\
        Reward Clip Range & [-100, 100] \\
        Dropout & 0.0 \\
        Warmup Steps & 5,000 \\
        Warmup Schedule Linear &(from 0 to LR) \\
        Learning Rate~(LR) & 0.00001\\
        Max Input Length & 1024 (Xsum) / 80 (Europarl) / 200 (GSM8K)\\
        Max Output Length & 64  (Xsum) / 80 (Europarl) / 300 (GSM8K)\\
        \bottomrule
    \end{tabularx}
    }
    \caption{Hyperparameter Details for experiments on RL-based approaches (\name{}, LLMR, LLMR with mean baseline, and LLMR with min-variance baseline).}\label{tab:hyp_rl}
\end{table}

\begin{table}[!b]
    \centering
    \resizebox{1\columnwidth}{!}{
    \begin{tabularx}{0.65\textwidth}{lX}
        \toprule
        \textbf{Hyperparameter} & \textbf{Value} \\
        \midrule
        Training Epochs & 2 \\
        Train Batch size & 32 \\
        Eval Batch size & 32 \\
        Optimizer & AdamW \\
        Grad Accumulation Steps & 16 \\
        Dropout & 0.25 \\
        Warmup Steps & 5,000 \\
        Warmup Schedule Linear &(from 0 to LR) \\
        Learning Rate~(LR) & 0.00005\\
        Max Input Length & 1024 (Xsum) / 80 (Europarl) / 200 (GSM8K)\\
        Max Output Length & 64  (Xsum) / 80 (Europarl) / 300 (GSM8K)\\
        \bottomrule
    \end{tabularx}
    }
    \caption{Hyperparameter Details for experiments on divergence-based KD approaches (seqKD, KL, JSD, TVD).}\label{tab:hyp_div}
\end{table}

\section{Results on More Models}
KD studies on seq2seq tasks have largely centered on encoder-decoder structures such as T5~\cite{t52020,chung2024scaling} and BART~\cite{lewis2020bart} models ~\cite{wen-etal-2023-f,li-etal-2024-llmr,agarwal2024onpolicy,jung2024impossible,wang2025distilling}. To evaluate the performance of our \name{} on recent popular decoder-only architectures, we also applied it to the Qwen1.5 model series ~\cite{qwen1.5} and report the results in Table ~\ref{tab:qwen-results}. The results are consistent with Table \ref{tab:main-result}, showing that our \name{} are robust to the shift of LM architectures.
\input{latex/Qwen_result}

\section{The Use of Large Language Models (LLMs)}
Gemini2.5~\cite{comanici2025gemini} was used in a limited capacity to improve writing quality, including checking grammar and rephrasing certain expressions with better sentence structures. In addition,  we use it for formatting LaTeX tables and Matplotlib figures.
However, we came up with the research ideas, conducted the analyses, and presented the contents without using AI tools.

\onecolumn

\section{Prompts Templates for LLM Evaluation}~\label{apd:llm_eval_prompt}

Table \ref{tab:llm_prompt_summarization} and Table \ref{tab:llm_prompt_translation} present our prompts template for LLM evaluation on the summarization task and machine translation task, respectively.
\begin{table}[!h]
\centering
\resizebox{1\columnwidth}{!}{
\begin{tabular}{@{}l@{}}
\toprule
Please evaluate the overall quality of the following summaries given the document. \\
\\
Evaluation Criteria: \\
Overall Quality: A good summary should be both precise and concise, summarizing the most
important points in the given document, \\ without including unimportant or irrelevant details \\
\\
Document: \textbf{[Source]} \\
Summary \textbf{[ID1]}: \textbf{[Summary-A]} \\
Summary \textbf{[ID2]}:  \textbf{[Summary-B]} \\
\\
FIRST, provide a one-sentence comparison of the two summaries for overall quality, explaining which you prefer and why.\\
SECOND, on a new line, state only the ID to indicate your choice. Your response should use the format: \\
Overall Quality: <one-sentence comparison and explanation> \\
Preferred: <summary ID> \\
\midrule
Please evaluate the informativeness of the following summaries given the document. \\
\\
Evaluation Criteria: \\
Informativeness: Does it include the most important details while excluding irrelevant content? \\
\\
Document: \textbf{[Source]} \\
Summary \textbf{[ID1]}: \textbf{[Summary-A]} \\
Summary \textbf{[ID2]}:  \textbf{[Summary-B]} \\
\\
FIRST, provide a one-sentence comparison of the two summaries for informativenss, explaining which you prefer and why.\\
SECOND, on a new line, state only the ID to indicate your choice. Your response should use the format: \\
Informativeness: <one-sentence comparison and explanation> \\
Preferred: <summary ID> \\
\midrule
Please evaluate the coherence of the following summaries given the document. \\
\\
Evaluation Criteria: \\
Coherence: Is the summary logically structured and easy to follow? \\
\\
Document: \textbf{[Source]} \\
Summary \textbf{[ID1]}: \textbf{[Summary-A]} \\
Summary \textbf{[ID2]}:  \textbf{[Summary-B]} \\
\\
FIRST, provide a one-sentence comparison of the two summaries for coherence, explaining which you prefer and why.\\
SECOND, on a new line, state only the ID to indicate your choice. Your response should use the format: \\
Informativeness: <one-sentence comparison and explanation> \\
Preferred: <summary ID> \\
\bottomrule
\end{tabular}}
\caption{Prompt templates for LLM evaluation on the summarization task in terms of overall quality, informativeness, and coherence. Here, ``\textbf{Source}'' is the document to be summarized. The choices of IDs are ``A'' and ``B''; ``\textbf{Summary-A}'' and ``\textbf{Summary-B}'' are replaced with model-generated texts. Since LLMs are not robust to ID and order~\citep{zheng2023large,shen2023large}, we enumerate different combinations for a given pair, resulting in four LLM queries.}
\label{tab:llm_prompt_summarization}
\end{table}

\begin{table}[!t]
\centering
\resizebox{1\columnwidth}{!}{
\begin{tabular}{@{}l@{}}
\toprule
Please evaluate the overall quality of the following translations from English to Dutch. \\
\\
Evaluation Criteria: \\
Overall Quality: A good translation should: 1) faithfully reflect the meaning of the source text; 2) avoid adding unnecessary or irrelevant \\ details. 3) use natural and fluent Dutch.\\
\\
Source: \textbf{[Source]} \\
Translation \textbf{[ID1]}: \textbf{[Translation-A]} \\
Translation \textbf{[ID2]}:  \textbf{[Translation-B]} \\
\\
FIRST, provide a one-sentence comparison of the two translations for overall quality, explaining which you prefer and why.\\
SECOND, on a new line, state only the ID to indicate your choice. Your response should use the format: \\
Overall Quality: <one-sentence comparison and explanation> \\
Preferred: <translation ID> \\
\midrule
Please evaluate the informativeness of the following translations from English to Dutch. \\
\\
Evaluation Criteria: \\
Informativeness: Does the translation preserve all key information without adding irrelevant details? \\
\\
Source: \textbf{[Source]} \\
Translation \textbf{[ID1]}: \textbf{[Translation-A]} \\
Translation \textbf{[ID2]}:  \textbf{[Translation-B]} \\
\\
FIRST, provide a one-sentence comparison of the two translations for informativeness, explaining which you prefer and why.\\
SECOND, on a new line, state only the ID to indicate your choice. Your response should use the format: \\
Informativeness: <one-sentence comparison and explanation> \\
Preferred: <translation ID> \\
\midrule
Please evaluate the coherence of the following translations from English to Dutch. \\
\\
Evaluation Criteria: \\
Coherence: Is the translation fluent, logically structured, and easy to understand in Dutch? \\
\\
Source: \textbf{[Source]} \\
Translation \textbf{[ID1]}: \textbf{[Translation-A]} \\
Translation \textbf{[ID2]}:  \textbf{[Translation-B]} \\
\\
FIRST, provide a one-sentence comparison of the two translations for coherence, explaining which you prefer and why.\\
SECOND, on a new line, state only the ID to indicate your choice. Your response should use the format: \\
Informativeness: <one-sentence comparison and explanation> \\
Preferred: <translation ID> \\
\bottomrule
\end{tabular}}
\caption{Prompt templates for LLM evaluation on the machine translation task in terms of overall quality, informativeness, and coherence. Here, ``\textbf{Source}'' is the source sentence to be translated. The choices of IDs are ``A'' and ``B''; ``\textbf{Translation-A}'' and ``\textbf{Translation-B}'' are replaced with model-generated texts. We still enumerate different combinations for a given pair, resulting in four LLM queries.}
\label{tab:llm_prompt_translation}
\end{table}

\end{document}

%% file: latex/algorithm.tex
\begin{algorithm}[!t]

\caption{\textsc{Ketchup}} \label{alg:rl_kstep}
\KwIn{\hangindent=-1em Non-parallel dataset $D$; teacher Q-value function $q: \mathcal{S}\times\mathcal{A}\to\mathbb{R}$; student policy $\pi_{\theta}$ with initial parameters $\theta$; segment length $K$; learning rate $\eta$; maximum rollout length $T$; number of iterations $U$}
\KwOut{Trained student policy $\pi_{\theta}$}
\For{$j \gets 1$ \KwTo $U$}{
    \!\!\!\quad Sample a source sentence $\mathbf{x}\in D$\\
    Set the initial state $s_0\gets \mathbf{x}$\\
    Generate a trajectory 
      $\tau = \{(s_0,a_0),\,(s_1,a_1)\\,\,\dots,\,(s_T,a_T)\}$
    by sampling from $\pi_{\theta}$\\
    Initialize gradient accumulator: $g\gets 0$\\
    \For{$t \gets T$ \KwTo $0$}{
         \uIf{$t = T$}{
              \small{$\hat{G}_T\gets q(s_T,a_T)$}\;
         }
         \uElseIf{$T-t < k$}{
              \small{$\hat{G}_t\gets\Bigl[q(s_t,a_t)-\max\limits_{a'\in\mathcal{A}}\,q(s_{t+1},a')\Bigr] \!+ \!\hat{G}_{t+1}$}\;
         }
         \Else{
              \small{$\hat{G}_t\gets\Bigl[q(s_t,a_t)-\max\limits_{a'\in\mathcal{A}}\,q(s_{t+K},a')\Bigr] \!+ \!\hat{G}_{t+K}$}\;
         }
         \small{$g\gets g+\hat{G}_t\,\nabla_{\theta}\log\pi_{\theta}(a_t\mid s_t)$}\;
    }
    \small{$\theta\gets \theta+\eta\,g$}\;
}
\Return{$\pi_{\theta}$}\;
\end{algorithm}

%% file: latex/vol_reward.tex
\begin{figure*}[ht]
    \centering
    \vspace*{-3pt}
    \includegraphics[width=1\textwidth]{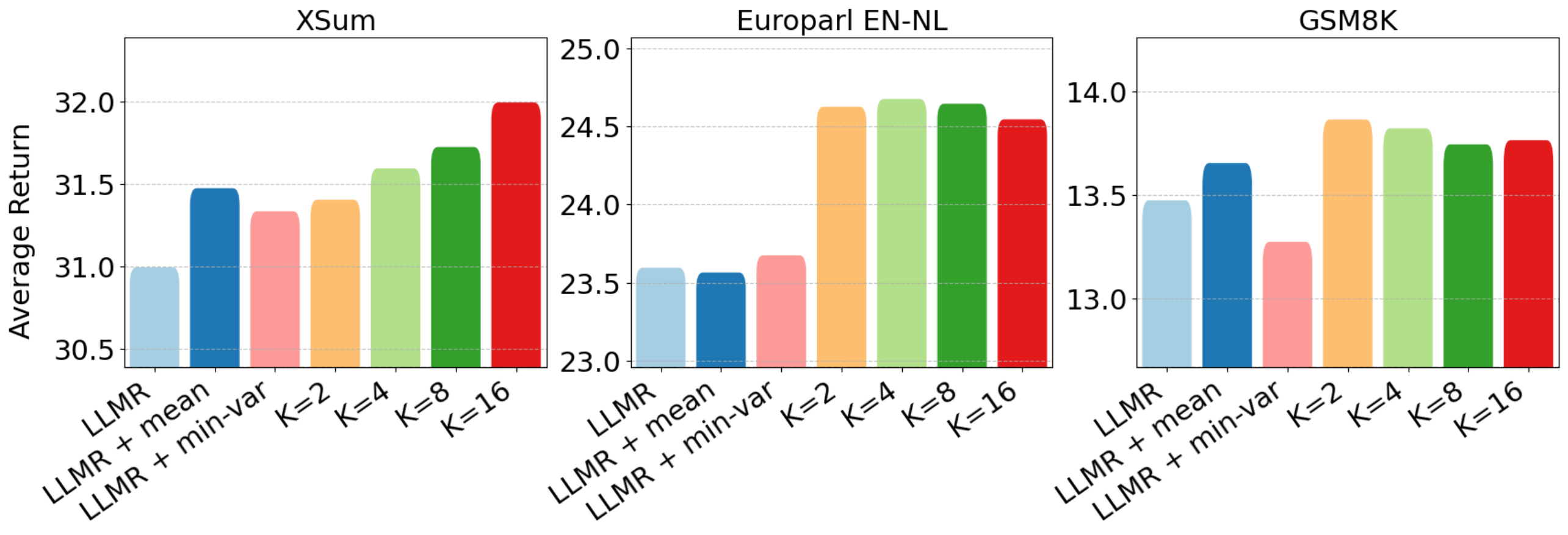}
    \caption{Average predicted return of different approaches.}
    \label{fig:vol_plot}
    
\end{figure*}

%% file: latex/main_results.tex
\renewcommand{\name}{\textsc{Ketchup}}
\begin{table*}[htbp]

\centering
\resizebox{\textwidth}{!}{
\begin{tabular}{|ll|ccc|ccc|c|}
\hline
\multicolumn{2}{|l|}{\multirow{2}{*}{Model}} &
  \multicolumn{3}{c|}{XSum} &
  \multicolumn{3}{c|}{Europarl EN-NL} &
  \multicolumn{1}{c|}{GSM8K} \\ \cline{3-9} 
\multicolumn{2}{|l|}{} &
  \multicolumn{1}{c}{ROUGE-1$^\uparrow$} &
  \multicolumn{1}{c}{ROUGE-2$^\uparrow$} &
  \multicolumn{1}{c|}{ROUGE-L$^\uparrow$} &
  \multicolumn{1}{c}{BLEU4$^\uparrow$} &
  \multicolumn{1}{c}{chrF$^\uparrow$} &
  \multicolumn{1}{c|}{TER$^\downarrow$} &
  \multicolumn{1}{c|}{Accuracy(\%)$^\uparrow$} \\ \hline
\multicolumn{2}{|l|}{Teacher} &
  41.32 & 18.86 & 33.79 & 25.36 & 51.11 & 63.17 & 40.71 \\ \hline
\multicolumn{2}{|l|}{Student} &
  19.60 & 3.19 & 13.72 & 0.95 & 24.80 & 100.21 & 0.00 \\ \hline
\multicolumn{1}{|l|}{\multirow{11}{*}{\shortstack{Distilled\\Student}}} &
  SeqKD~\cite{kim-rush-2016-sequence} &
  33.54 & 11.90 & 26.67 & 22.09 & 48.33 & 66.18 & 20.02 \\
\multicolumn{1}{|l|}{} &
  KL~\cite{hinton2015distillingknowledgeneuralnetwork} &
  34.36 & 12.86 & 27.38 & 22.35 & 48.58 & 65.93 & 23.96 \\
\multicolumn{1}{|l|}{} &
  JS \cite{wen-etal-2023-f} &
  34.87 & 13.18 & 27.84 & 22.55 & 48.71 & 65.74 & 24.72 \\ 
\multicolumn{1}{|l|}{} &
  TVD \cite{wen-etal-2023-f}&
  35.17 & 13.30 & 28.10 & 22.63 & 48.66 & 65.79 & 24.94\\ \cline{2-9}
\multicolumn{1}{|l|}{} &
  LLMR~\cite{li-etal-2024-llmr} &
  35.54 & 13.70 & 28.56 & 22.72 & 49.04 & 65.38 & 25.21 \\ 

\multicolumn{1}{|l|}{} &
  LLMR + Mean baseline &
  35.60 & 13.76 & 28.64 & 22.67 & 49.03 & 65.39 & 25.39 \\ 
\multicolumn{1}{|l|}{} &
 LLMR + Min-Var baseline &
35.59 & 13.78 & 28.66 & 22.70 & 48.97 & 65.55 & 25.10\\ \cline{2-9}
\multicolumn{1}{|l|}{} &
  \name{} ($K=2$) &
  \textbf{36.03} & 13.95 & \textbf{28.89} & 22.93 & \textbf{49.25} & \textbf{65.15} & 25.32 \\ 
\multicolumn{1}{|l|}{} &
  \name{} ($K=4$) &
  35.96 & \textbf{13.96}& 28.87 & 22.93 & 49.21 & 65.21 & 25.40 \\ 
\multicolumn{1}{|l|}{} &
  \name{} ($K=8$) &
  35.68 & 13.88 & 28.76 & \textbf{22.95} & 49.23 & 65.20 & \textbf{25.71} \\
\multicolumn{1}{|l|}{} &
  \name{} ($K=16$) &
  35.31 & 13.68 & 28.51 & 22.94& 49.24 & 65.18 & 25.47 \\ \hline
\end{tabular}
}

\caption{Main results on XSum, Europarl EN--NL, and GSM8K datasets.
The best student result is in \textbf{bold}.
$^{\uparrow/\downarrow}$The higher/lower, the better. We prompt the teacher and off-the-shelf student in a zero-shot manner to gain the first two rows. We select the best checkpoint based on the performance of the held-out validation set and report the performance of these checkpoints on the test set for all distilled students.
}\label{tab:main-result}

\end{table*}

%% file: latex/variance.tex
\begin{figure}[!t]
    \centering
    \begin{subfigure}[b]{0.49\columnwidth}
        \includegraphics[width=\textwidth]{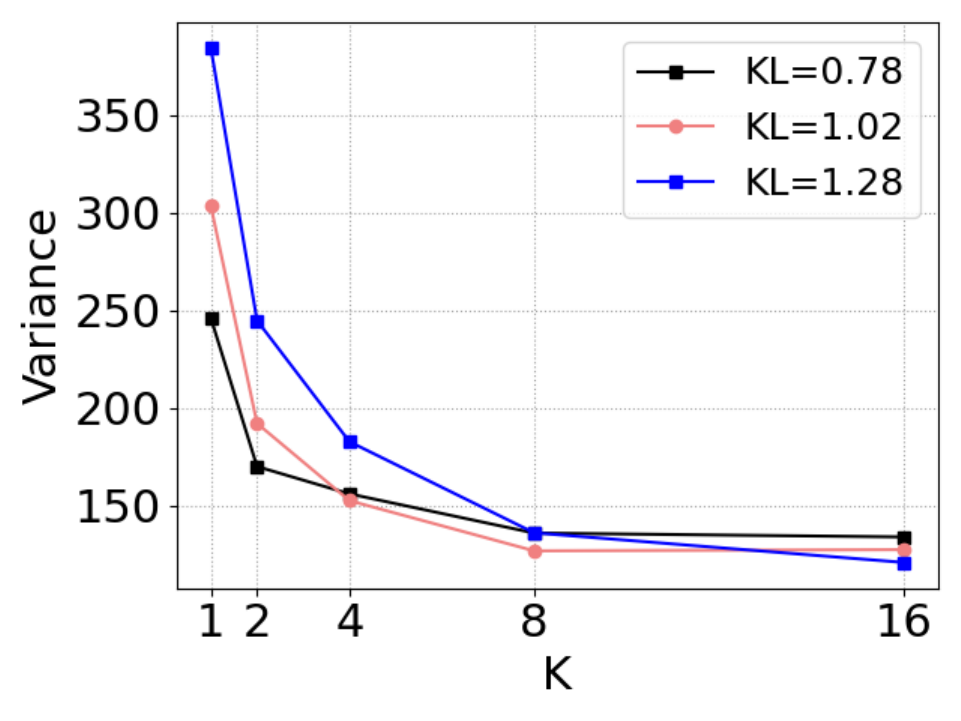}
        \caption{Variance} 
        \label{fig:variance}
    \end{subfigure}
    \hfill
    \begin{subfigure}[b]{0.49\columnwidth}
        \includegraphics[width=\textwidth]{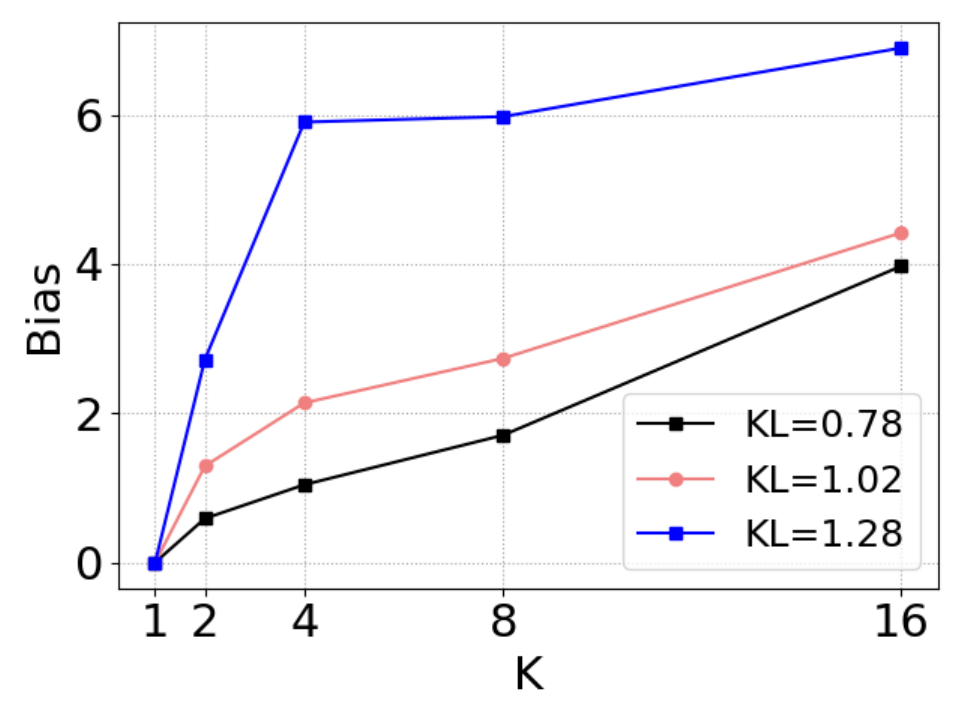}
        \caption{Bias} 
        \label{fig:bias}
    \end{subfigure}
    \caption{Variance and bias with different $K$ values.} 
\end{figure}

%% file: latex/plot_model_size.tex
\begin{figure*}[ht]
    \centering
    \vspace*{-3pt}
    \begin{tabular}{c}     \includegraphics[width=1.0\textwidth] {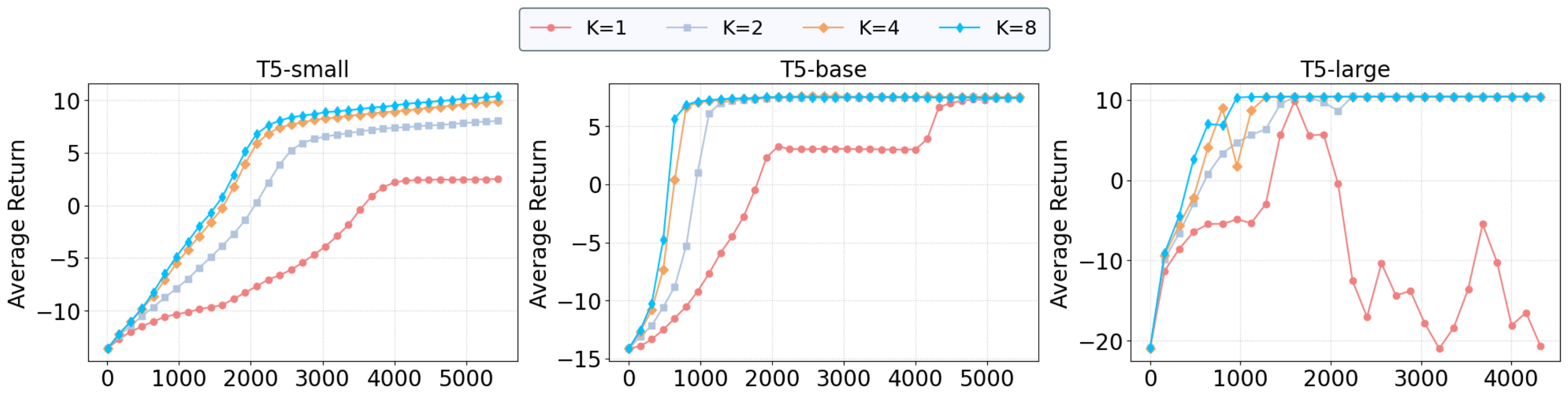} 
    \end{tabular}
 
    \caption{Learning curves with different $K$ values and model sizes, where the $x$-axis is the number of training steps.}
    \label{fig:plot_model_size}
\end{figure*}

%% file: latex/dataset_stats.tex
\begin{table}[!t]
\centering
\resizebox{1\columnwidth}{!}{
\begin{tabular}{|l|l|rrr|}
\hline
\multicolumn{1}{|c|}{\multirow{2}{*}{Dataset}} &
  \multicolumn{1}{c|}{\multirow{2}{*}{Task}} &
  \multicolumn{3}{c|}{\# of Samples} \\ \cline{3-5} 
\multicolumn{1}{|c|}{} &
  \multicolumn{1}{c|}{} &
  \multicolumn{1}{c}{Train} &
  \multicolumn{1}{c}{Dev} &
  \multicolumn{1}{c|}{Test} \\ \hline
XSum~\citep{narayan-etal-2018-xsum} &
  Summarization &
  202,926 &
  11,332 &
  11,333 \\
Europarl EN-NL~\citep{koehn2005europarl} &
  Machine Translation &
  1,167,808 &
  10,014 &
  10,016 \\
GSM8K~\citep{cobbe2021training} &
  Arithmetic reasoning &
  6,705 &
  768 &
  1,319 \\ \hline
\end{tabular}
}
\caption{Statistics of our datasets.}
\label{tab:exp-details}
\end{table}

%% file: latex/Qwen_result.tex
\renewcommand{\name}{$K$\textsc{etchup}}
\begin{table}[!h]
  \centering
  \scriptsize
  \setlength{\tabcolsep}{3pt}  
  \resizebox{\columnwidth}{!}{%
    \begin{tabular}{|l|c|c|c|}
      \hline
      Model 
        & XSum (ROUGE-1$\uparrow$) 
        & Europarl (BLEU4$\uparrow$)
        & GSM8K (Acc.\,(\%)$\uparrow$) \\ 
      \hline
      Teacher (Qwen1.5-4B)                                   & 38.15 & 21.32 & 42.08 \\ 
      Student (Qwen1.5-0.5B)                                 & 8.80 &  0.02 &  0.00 \\ \hline
      KL ~\cite{hinton2015distillingknowledgeneuralnetwork}  & 31.29 & 15.76 & 26.31 \\
      TVD ~\cite{wen-etal-2023-f}                            & 31.18 & 16.22 & 26.99 \\ 
      LLMR ~\cite{li-etal-2024-llmr}                         & 31.61 & 15.90 & 27.29 \\ \hline
      \name                 & \textbf{32.28} & \textbf{16.46} & \textbf{28.13} \\

      \hline
    \end{tabular}%
  }
  \caption{Distillation results on XSum, Europarl EN–NL, and GSM8K using Qwen1.5 models.
           Higher $\uparrow$ is better. The best $K$ values are 2, 2, and 16 for the three datasets, respectively.}
  \label{tab:qwen-results}
\end{table}